\documentclass[letterpaper, 10 pt, conference]{ieeeconf}
\IEEEoverridecommandlockouts   % if \thanks is used
\usepackage{graphicx}  % for pdf, bitmapped graphics files
\usepackage[tbtags]{amsmath}   % assumes amsmath package installed
\usepackage{amssymb}   % assumes amsmath package installed
\usepackage{microtype} % microtypography
\usepackage{cite}
\usepackage{mathtools}
\usepackage{cleveref}
\usepackage[dvipsnames]{xcolor}
\usepackage{soul}
\usepackage{bm}
\usepackage[caption=false,font=footnotesize]{subfig}
\usepackage{siunitx}
\graphicspath{{figures/}{../figures/}}

\newtheorem{assumption}{Assumption}
\newtheorem{theorem}{Theorem}[section]

\crefname{section}{Sec.}{sections}
\crefname{figure}{Fig.}{figures}
\crefname{table}{Table}{tables}
\crefformat{equation}{(#2#1#3)}
\crefrangeformat{equation}{(#3#1#4) to~(#5#2#6)}
\crefmultiformat{equation}{(#2#1#3)}%
{ and~(#2#1#3)}{, (#2#1#3)}{ and~(#2#1#3)}
\crefrangemultiformat{equation}{(#3#1#4) to~(#5#2#6)}%
{ and~(#3#1#4) to~(#5#2#6)}{, (#3#1#4) to~(#5#2#6)}{ and~(#3#1#4) to~(#5#2#6)}
\crefname{assumption}{Assumption}{Assumptions}

\newcommand{\yaxis}{$y$\nobreakdash-axis}
\newcommand{\zaxis}{$z$\nobreakdash-axis}

\newcommand{\xzplane}{$xz$\nobreakdash-plane}

\newcommand{\B}[1]{\bm{#1}}
\newcommand{\C}[1]{\mathcal{#1}}

\newcommand{\expnumber}[2]{{#1}\mathrm{e}{#2}}
\newcommand{\norm}[1]{\bigl\lVert#1\bigr\rVert}
\newcommand{\diag}{\mathop{\mathrm{diag}}}

\newcommand{\ctx}{\bar{C}_{Tx}}
\newcommand{\ctz}{\bar{C}_{Tz}}
\newcommand{\cfs}{\bar{C}_{S}}
\newcommand{\cl}{C_L}
\newcommand{\clo}{C_{L_0}}
\newcommand{\cla}{C_{L_1}}
\newcommand{\cd}{C_D}
\newcommand{\cdo}{C_{D_0}}
\newcommand{\cda}{C_{D_1}}
\newcommand{\cdaa}{C_{D_2}}
\newcommand{\cy}{C_Y}

\newcommand{\inertia}{\B{J}}

\newcommand{\pos}{\B{p}}
\newcommand{\dpos}{\dot{\pos}}
\newcommand{\pd}{\B{p}_d}
\newcommand{\dpd}{\dot{\B{p}}_d}
\newcommand{\perr}{\tilde{\B{p}}}

\newcommand{\vel}{\B{v}}
\newcommand{\dvel}{\dot{\vel}}
\newcommand{\vr}{\B{v}_r}
\newcommand{\dvr}{\dot{\B{v}}_r}
\newcommand{\verr}{\tilde{\B{v}}}
\newcommand{\dverr}{\dot{\verr}}

\newcommand{\Rot}{\B{R}}
\newcommand{\dRot}{\dot{\Rot}}
\newcommand{\Rd}{\B{R}_d}
\newcommand{\Rerr}{\tilde{\B{R}}}
\newcommand{\qerr}{\tilde{\B{q}}_{v}}

\newcommand{\qoerr}{\tilde{q}_0}
\newcommand{\dqoerr}{\dot{\tilde{q}}_0}

\newcommand{\angular}{\bm{\omega}}
\newcommand{\dangular}{\dot{\angular}}
\newcommand{\angulard}{\bm{\omega}_d}
\newcommand{\angularref}{\bm{\omega}_r}
\newcommand{\dangularref}{\dot{\bm{\omega}_r}}
\newcommand{\angularerr}{\tilde{\angular}}

\newcommand{\fr}{\bar{\B{f}}}
\newcommand{\frb}{\fr_b}
\newcommand{\mr}{\bar{\bm{\tau}}}

\newcommand{\gravity}{\B{g}}
\newcommand{\fb}{\B{f}_b}
\newcommand{\fbhat}{\hat{\B{f}}_b}

\newcommand{\mb}{\bm{\tau}_b}

\newcommand{\mask}{\B{M}}
\newcommand{\fT}{\B{f}_T}
\newcommand{\fThat}{\hat{\B{f}}_T}
\newcommand{\fA}{\B{f}_A}
\newcommand{\fAhat}{\hat{\B{f}}_A}
\newcommand{\thhat}{\hat{\bm{\theta}}}
\newcommand{\thinit}{\bm{\theta}_0}
\newcommand{\dthhat}{\dot{\thhat}}
\newcommand{\therr}{\tilde{\bm{\theta}}}
\newcommand{\ferr}{\bm{\zeta}}

\newcommand{\vi}{\B{v}_i} % incident velocity
\newcommand{\vw}{\B{v}_w} % wind velocity
\newcommand{\ix}{\hat{\B{x}}_i}
\newcommand{\iy}{\hat{\B{y}}_i}
\newcommand{\iz}{\hat{\B{z}}_i}
\newcommand{\Ri}{\B{R}_i}
\newcommand{\Ralpha}{\B{R}_{\alpha}}

\newcommand{\cov}{\B{P}}
\newcommand{\basis}{\bm{\Phi}}
\newcommand{\basisf}{\B{W}}
\newcommand{\prederr}{\B{e}}
\newcommand{\am}{\B{a}_m}
\newcommand{\forgetting}{\lambda}
\newcommand{\robustify}{\sigma}

\newcommand{\vinf}{V_\infty}
\newcommand{\sref}{S_\text{ref}}

\title{\LARGE \bf
Adaptive Nonlinear Control of Fixed-Wing VTOL \\with Airflow Vector Sensing}

\author{Xichen Shi$^{1}$, Patrick Spieler$^{1}$, Ellande Tang$^{1}$, Elena-Sorina Lupu$^{1}$, Phillip Tokumaru$^{2}$, and Soon-Jo Chung$^{1}$
\thanks{$^{1}$Xichen Shi, Patrick Spieler, Ellande Tang, Elena-Sorina Lupu and Soon-Jo Chung are with California Institute of Technology. \texttt{\{xshi, pspieler, ellandet, eslupu, sjchung\}@caltech.edu}.}
\thanks{$^{2}$Phil Tokumaru is with AeroVironment, Inc.
\texttt{tokumaru@avinc.com}.}
\thanks{The authors thank M. Gharib for his technical guidance. This work is in part funded by AeroVironment, Inc.}
}

\begin{document}

\maketitle
\thispagestyle{empty}
\pagestyle{empty}

%%%%%%%%%%%%%%%%%%%%%%%%%%%%%%%%%%%%%%%%%%%%%%%%%%%%%%%%%%%%%%%%%%%%%%%%%%%%%%%%
\begin{abstract}
Fixed-wing vertical take-off and landing (VTOL) aircraft pose a unique control challenge that stems from complex aerodynamic interactions between wings and rotors. Thus, accurate estimation of external forces is indispensable for achieving high performance flight. In this paper, we present a composite adaptive nonlinear tracking controller for a fixed-wing VTOL. The method employs online adaptation of linear force models, and generates accurate estimation for wing and rotor forces in real-time based on information from a three-dimensional airflow sensor. The controller is implemented on a custom-built fixed-wing VTOL, which shows improved velocity tracking and force prediction during the transition stage from hover to forward flight, compared to baseline flight controllers.
\end{abstract}

%%%%%%%%%%%%%%%%%%%%%%%%%%%%%%%%%%%%%%%%%%%%%%%%%%%%%%%%%%%%%%%%%%%%%%%%%%%%%%%%

\section{Introduction}
Vertical take-off and landing (VTOL) aircraft have been an area of intense research for most of the past century. A variety of VTOL craft have been developed in the military and civilian spaces. The operational simplicity associated with not requiring a runway and being able to hover in place often outweigh the negative aspects of the design complexity. In recent years, improvements in battery technology, computing power, and sensor availability have spurred the development of multi-rotors.
% relatively small craft that typically achieve stable flight through the use of multiple electric motors with propellers. 
% The simplicity in the design, construction, and control of these platforms has made them very popular as a cheap and robust aerial platform. 
However, a standard multi-rotor lacks the efficiency for long range flight. Married with the trend in small-scale electric drones, the class of "hybrid" VTOL aircraft with both a fixed-wing and rotors has recently risen in popularity. They use lifting surfaces to enable longer range and endurance flights and keep VTOL capabilities, eliminating the need for a substantial runway. These types of vehicles are now re-branded as ``flying cars'' and are gaining traction commercially~\cite{uberelevate}.

Depending on the configuration of thrusters on the craft, fixed-wing VTOL can be categorized as tilt-rotor, tail-sitter, or copter-plane. Although different in geometry, the underlying control logic is similar. Most controllers designed for such crafts rely on two separate schemes, one for VTOL and one for fixed-wing. A transition strategy is designed to switch between the two. Because of the hybrid nature during this transition period, complex interactions between propellers and wings pose challenge for accurate and safe flight maneuvers. Early works achieve this transition by overlapping the operation envelope of the two controllers. In~\cite{frank2007hover,stone2008flight}, reference commands were sent to the VTOL controller such that the vehicle would either reach a high-speed or a low pitch angle state triggering the fixed-wing controller to become active. A common technique for tilt-rotor transition is to vary tilt angles following a monotonic schedule, during which the controller stabilizes the craft~\cite{chowdhury2012back,park2013fault}. However, little attention was given to aerodynamic and flight-dynamic modeling in this scenario. More recent methods utilize numerical optimization to solve for a trajectory based on accurate vehicle dynamics. \cite{verling2017model,oosedo2017optimal} solve for transition trajectories offline and then deploy feedback tracking controllers to execute them online. To further improve the performance, \cite{zhou2017unified,ritz2017global} propose online optimization-based controller with global aerodynamic model for tail-sitters. Such controllers can give solutions between any global states, as long as the onboard computer can solve the problem in real time.

\begin{figure}[!t]
\centering
\includegraphics[width=\linewidth]{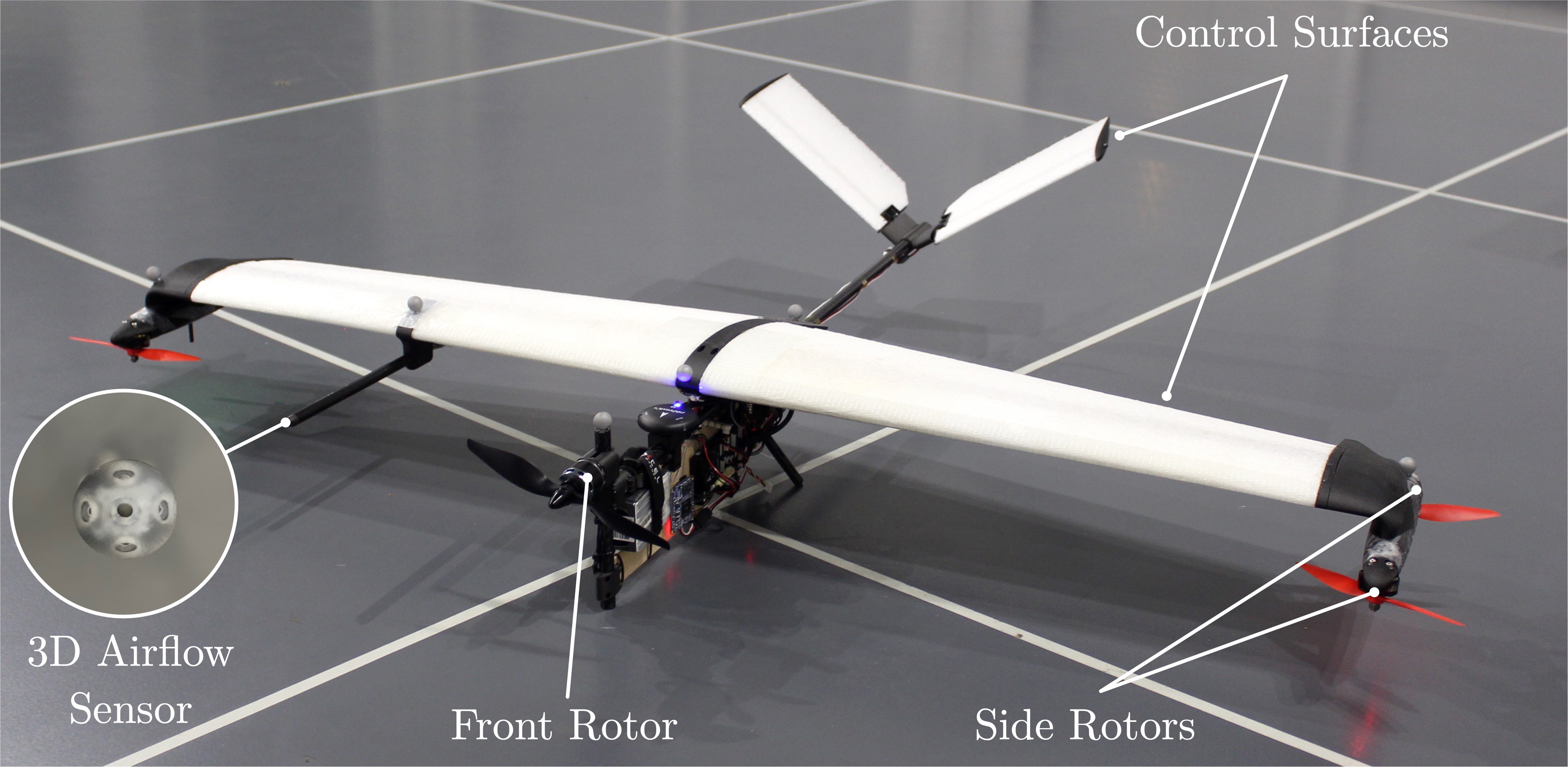}
\caption{Fixed-wing VTOL platform with wing and multi-rotor}
\label{fig:av_vtol}
\vspace{-4 mm}
\end{figure}

Despite such success, heavy computation requirements still put a burden on the craft, especially for small scale UAVs with limited payload and power. Furthermore, iterative solvers are often several orders of magnitude slower than a simple feedback controller. And execution speed can be critical for agile flyers. There has been some interest in designing a unified feedback controller for fixed-wing VTOL, such as in~\cite{pucci2013towards,bulka2018universal}, and our prior work~\cite{shi2018nonlinear}. The success of these controllers relies primarily on the accurate prediction of aerodynamic forces, which in turn requires high-fidelity models and accurate sensor feedback for states relevant to such forces. Surprisingly, there has been little work on these areas; complex aerodynamic interactions between wing and rotors that are crucial to fixed-wing VTOL transition are often ignored in the prior work. Although efforts have been made to estimate aerodynamic states such as angle-of-attack and sideslip angle~\cite{johansen2015estimation,tian2018model}, no recent work has been using such information directly in a feedback control manner. On the other hand, adaptive flight control has seen much progress over the years, where the vehicle model is adapted via either aerodynamic coefficients~\cite{farrell2005backstepping,gavilan2014adaptive} or neural network parameters~\cite{calise1998nonlinear,lee2001nonlinear}.

\emph{Contribution and Organization:} We build on the prior work in~\cite{shi2018nonlinear} where we proposed a general control architecture for fixed-wing VTOL. Here we focus on creating a velocity tracking controller with an adaptive force model and a novel 3D airflow sensor for accurate aerodynamic force prediction. The paper is organized as follows: In \cref{sec:system}, the general fixed-wing VTOL dynamics, linear force model, and control architecture are introduced. In \cref{sec:f_alloc}, we propose adaptive force allocation and prove convergence of tracking and prediction errors. Model fitting and the 3D airflow sensor are presented in \cref{sec:aero}. Experimental results comparing several control techniques are shown in \cref{sec:exp}. Lastly, concluding remarks are stated in \cref{sec:conclusion}.

\section{System Overview}
\label{sec:system}
We have constructed a small-scale fixed-wing VTOL shown in~\cref{fig:av_vtol}, with four rotors mounted vertically, a forward-facing thruster, an all-moving V-tail, and a wing surface.
%%%%%%%%%%%%%%%%%%%%%%%%%%%%%%%%%%%%%%%%%%%%%%%%%%%%%%%%%%%%%%%%%%%%%%%%%%%%%%%%
\subsection{Fixed-Wing VTOL Dynamic Model}
A six degree-of-freedom (DOF) dynamics model for VTOL aircraft is considered. The system states are defined by inertial position $\pos$ and velocity $\vel$; attitude as rotation matrix $\Rot \in SO(3)$; and angular velocity $\angular$ in the body frame. The dynamics are expressed as:
\begin{align}
    \dpos &= \vel &    \dvel &= \gravity + \Rot \fb \label{eq:pos_dyn} \\ 
    \dRot &= \Rot \B{S}(\angular) & \inertia \dangular &= \B{S}(\inertia \angular) \angular + \mb \label{eq:att_dyn}
\end{align}
where $\inertia \in \mathbb{R}^{3\times3}$ is the inertia matrix of the vehicle in body-frame, $\gravity$ is the constant gravity vector in the inertial frame and $\B{S}(\cdot): \mathbb{R}^3 \to SO(3)$ is a skew-symmetric mapping such that $\B{a} \times \B{b} = \B{S}(\B{a})\B{b}$. External forces and moments on the vehicle are grouped into $\fb$ and $\mb$. $\fb$ is normalized with mass and has units of \si{m/s^2}. We can decompose 
\begin{equation}
    \fb = \fT + \fA \label{eq:fb}
\end{equation}
since the body forces of flying vehicles generally consist of only thruster forces $\fT$ and body aerodynamic force $\fA$.
In this work, we focus solely on the force modeling and its use in precise velocity control of a fixed-wing VTOL through transition flight at various air speeds. The moment $\mb$ is assumed to be properly accounted for by the attitude controller. This is true for general VTOL aircraft with well-tuned inner-loop attitude control.

%%%%%%%%%%%%%%%%%%%%%%%%%%%%%%%%%%%%%%%%%%%%%%%%%%%%%%%%%%%%%%%%%%%%%%%%%%%%%%%%
\subsection{Force Models}
% force diagrams

Forces on the wing and rotor are determined by relative wind velocity. We define this incident wind velocity on the vehicle as $\vi$, which in body frame is expressed as
\begin{equation}
    \vi = \Rot^\top\left(\vel - \vw\right) \label{eq:vi}
\end{equation}
where $\vw$ is the external wind velocity in inertial frame. Given $\vi$ we can define the following quantities related to force calculations
\begin{equation}
    \vinf = \norm{\vi},\ \alpha = \arctan\left(\frac{v_{iz}}{v_{ix}}\right),\ \beta = \arcsin\left(\frac{v_{iy}}{\vinf}\right) \label{eq:alphabeta}
\end{equation}
which are incident wind speed $\vinf$, angle-of-attack $\alpha$, and sideslip angle $\beta$.

\begin{figure}
\centering
\subfloat[Propeller Force Diagram]{
        \includegraphics[height=0.8in]{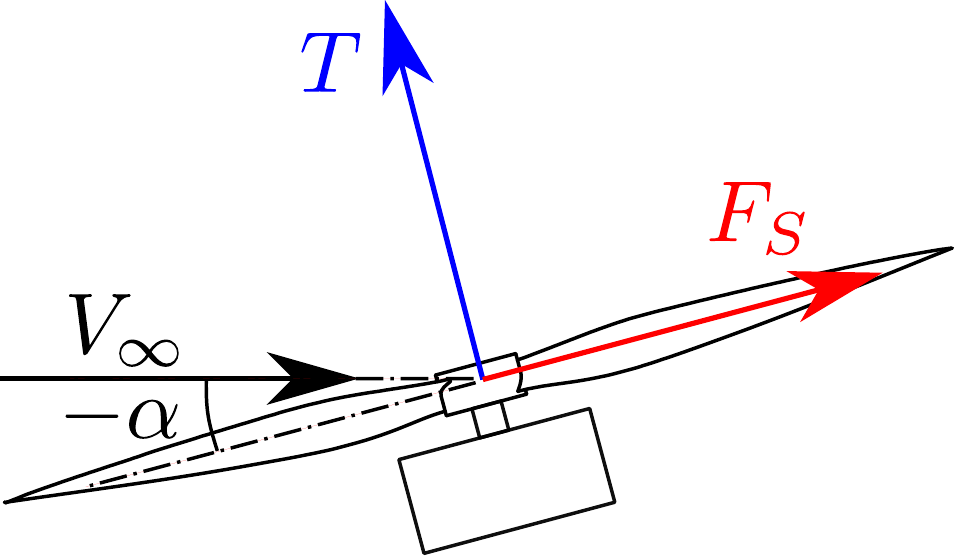}
        \label{fig:fT}
    }
\subfloat[Wing Force Diagram]{
        \includegraphics[height=0.8in]{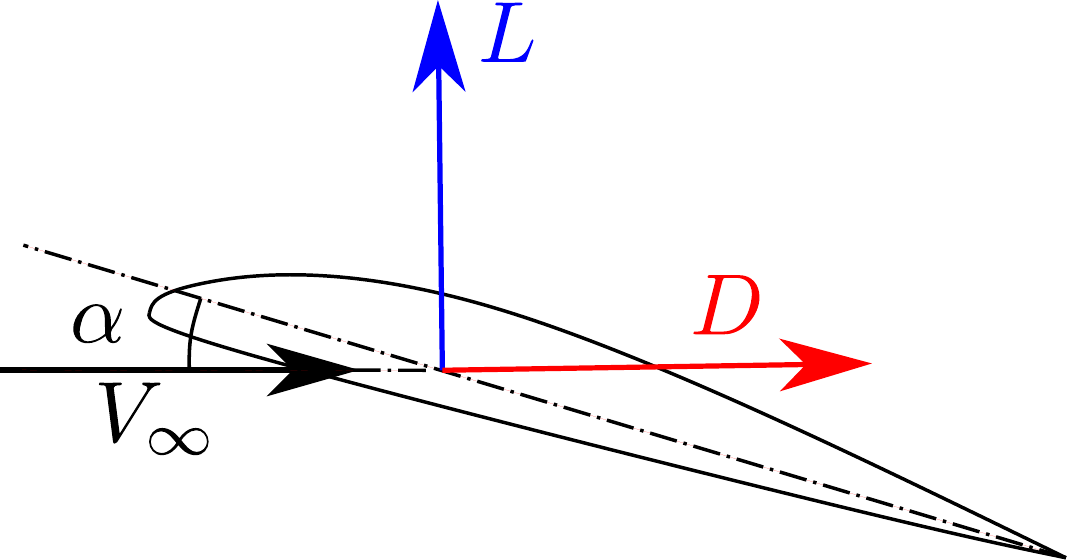}
        \label{fig:fA}
    }
\caption{Diagrams with definition of forces on different propeller and wing}
\vspace{-5 mm}
\end{figure}

As shown in~\cref{fig:fT}, typical propellers generate an axial thrust and experience a side force tangential to the rotor plane. A common model for static propeller thrust is
\begin{equation}
T = C_T \rho d^4 n^2 \label{eq:model_T}
\end{equation}
with non-dimensional thrust coefficient $C_T$, air density $\rho$, propeller diameter $d$ and propeller rotation speed $n$. The rotor side force $F_S$ however lacks simplified models, even though measurements indicate that the side forces produced by vertical thrusters often incur a significant drag increase during forward flight. Hence, we derive a canonical model from wind tunnel test data and dimensional analysis
\begin{equation}
    F_S = C_{S} \rho n^{k_1} \vinf^{2-k_1} d^{2+k_1}\left(\left(\frac{\pi}{2}\right)^2 - \alpha^2 \right) \left(\alpha + k_2\right) \label{eq:model_Fs}
\end{equation}
where $\alpha$ is in radians. Details about the model are described in~\cref{ssec:aerotest}. By assuming a linear relationship between $n$ and the throttle signal $u$, and that only the vertical thrusters produce significant side forces, we get the following force estimate
\begin{align}
    T_x &= \ctx u_x^2,\quad T_z = \ctz u_z^2, \label{eq:model_Txz} \\
    F_S &= \cfs \cdot G(\alpha, \vinf, u_z). \label{eq:model_FS}
\end{align}
The coefficients $\ctx$, $\ctz$, and $\cfs$ are non-dimensional coefficients combined with constants in~\cref{eq:model_T,eq:model_Fs}. $G(\alpha, \vinf, u_z)$ is the basis from \cref{eq:model_Fs} using fixed $k_1$ and $k_2$ fitted from data.

For aerodynamic forces on wing, we define lift $L$ and drag $D$ in body \xzplane \ shown in~\cref{fig:fA}, and side force $Y$ in body \yaxis. They are often non-dimensionalized with
\begin{equation*}
    L =\frac{1}{2}\rho \vinf^2\sref\cl, \
	D =\frac{1}{2}\rho \vinf^2\sref\cd, \
    Y =\frac{1}{2}\rho \vinf^2\sref\cy.
\end{equation*}
with $\sref$ as the reference wing area. A common linear aerodynamics model for $\cl$ and $\cd$ is
\begin{align}
    \cl &= \clo + \cla\alpha \label{eq:model_cl} \\
    \cd &= \cdo + \cda\alpha + \cdaa\alpha^2, \label{eq:model_cd}
\end{align}
while the side force $\cy$ remains unmodeled with the assumption that $\cy=0$ at $\beta=0$. This assumption is valid for most aircraft symmetric in the \xzplane. In~\cref{sec:f_alloc}, we use this fact to constrain $Y = 0 $ by ensuring $\beta=0$.

From~\cref{eq:model_Txz,eq:model_FS,eq:model_cl,eq:model_cd}, the estimated thruster and aerodynamic forces can be expressed as
\begin{align}
    \fThat &=
    \begin{bmatrix}
    \ctx u^2_x \\
    0  \\
    -\ctz u^2_z 
    \end{bmatrix}     
    +
    \begin{bmatrix}
    -\cfs G_z(\alpha, \vinf, u_z)\cos(\bar{\beta}) \\
    -\cfs G_z(\alpha, \vinf, u_z)\sin(\bar{\beta}) \\
    0
    \end{bmatrix}  \label{eq:fThat}\\
    \fAhat &=
        \frac{1}{2}\rho \sref  \vinf^2
    \begin{bmatrix}
    \cl(\alpha) \sin\alpha - \cd(\alpha) \cos\alpha \\
    -\cy(\beta) \\
    -\cl(\alpha) \cos\alpha - \cd(\alpha)\sin\alpha
    \end{bmatrix}.
    \label{eq:fAhat}
\end{align}
Note that $\bar{\beta} = \arctan\left(v_{iy}/v_{iz}\right)$ is slightly different from sideslip. We can express the overall body force estimate
\begin{equation*}
    \fbhat = \fThat + \fAhat = \basis (\vi, u_x, u_z) \thhat
\end{equation*}
where $\basis$ denotes the model basis collected from~\cref{eq:model_cl,eq:model_cd,eq:fThat,eq:fAhat}. The model parameter vector is
$$\thhat= \left[\ctx,\ \ctz,\ \cfs,\ \cdo,\ \cda,\ \cdaa,\ \clo,\ \cla \right]^\top.$$

%%%%%%%%%%%%%%%%%%%%%%%%%%%%%%%%%%%%%%%%%%%%%%%%%%%%%%%%%%%%%%%%%%%%%%%%%%%%%%%%
\subsection{Control Architecture}
In~\cite{shi2018nonlinear}, we proposed a control architecture for general VTOL aircraft subject to aerodynamic effects. It uses total body forces and moments as inputs to track 3-dimensional position trajectory $\pd(t)$. We define  position error as $\perr = \pos - \pd$. Assume attitude trajectory $\Rd(t)$, $\angulard(t)$ is defined from $\pd(t)$. We can specify the reference velocity $\vr$ and angular rate $\angularref$ as
\begin{align}
    \vr = \dpd - \B{{\Lambda}}_p \perr,\qquad
    \angularref = \Rerr^\top \angulard - \B{{\Lambda}}_q \qerr \label{eq:wr}
\end{align}
where $\B{{\Lambda}}_p$ and $\B{{\Lambda}}_q$ are positive definite gain matrices for position and quaternion errors. Thus, the force and moment controllers become
\begin{align}
    \fr &= -\gravity + \dvr - \B{K}_v \verr \label{eq:f_ctrl}\\
    \mr &= \inertia \dangularref -\B{S}\left(\inertia \angular\right)\angularref - \B{K}_{\omega} \angularerr, \label{eq:m_ctrl}
\end{align}
with error $\verr = \vel - \vr$ and $\angularerr = \angular - \angularref$, where (\ref{eq:m_ctrl}) is from \cite{bandyopadhyay2016nonlinear}. The attitude error expressed as the rotation matrix $\Rerr = \Rd^\top \Rot$ can be converted to a constrained quaternion error $\tilde{\B{q}} = [\qoerr, \qerr]$ with $\qoerr > 0$~\cite{shi2018nonlinear}.

The controllers in~\cref{eq:f_ctrl,eq:m_ctrl} are written in general form for any aircraft type, but each different type relies on $\fr$ and $\mr$ achieved through varying actuator commands or vehicle attitudes. We name such process as force allocation~\cite{shi2018nonlinear}, as shall be seen in the following section.

\section{Adaptive Force Allocation}
\label{sec:f_alloc}
The reference force command is achieved through a two-step process, solving for desired attitude and thrust commands respectively. First, desired attitude $\Rd$ is obtained through
\begin{equation}
    \Rd \cdot \left(\fThat(\vi, u_x, u_z) + \fAhat(\vi) \right)_d = \fr
    \label{eq:fbar_solve}
\end{equation} 
with $\vi$ also a function of $\Rd$ as seen from~\cref{eq:vi}. Second, we solve for the thrust commands by minimizing the force residue $\norm{\Rot \left(\fThat+\fAhat\right) - \fr}$ in current frame. 

%%%%%%%%%%%%%%%%%%%%%%%%%%%%%%%%%%%%%%%%%%%%%%%%%%%%%%%%%%%%%%%%%%%%%%%%%%%%%%%%
\subsection{Desired Attitude Determination}
A solution to~\cref{eq:fbar_solve} for a multicopter relies on the differential flatness property, with $\fAhat=0$ or considering only drag effect~\cite{mellinger2011minimum,faessler2017differential}. For a fixed-wing aircraft, a similar property can be derived by constraining the vehicle in the coordinated flight frame~\cite{hauser1997aggressive}. We propose a simplified solution following the same procedure in~\cite{shi2018nonlinear}, with the assumption of small $\alpha$ in desired flight condition. 

Let $\frb = \Rot^\top \fr$ denote the desired body frame force. Given measured $\vi$, the incident wind frame axes are defined as
\begin{equation*}
\ix = \frac{\vi}{\norm{\vi}}, \ \ \iy = \frac{\vi\times \frb}
{\norm{\vi\times \frb}},\ \
\iz= \ix \times \iy.
\end{equation*}
Incident wind frame is then represented in body frame as
\begin{equation*}
    \Ri= [\ix, \iy, \iz].
\end{equation*}
in which no side force is required (i.e $\frb \cdot \iy = 0$). Wing lift is then prioritized to produce $\frb$ normal to $\vi$ by setting desired angle of attack $\alpha_d$ according to the lift model
\begin{equation}
\alpha_d =
\begin{cases}
 \frac{-\frb\cdot\iz - L_0}{L_1} &\text{if} \quad -\frb\cdot\iz \leq L_{\max}\\
 \alpha_{\max} &\text{if} \quad  -\frb\cdot\iz > L_{\max} 
\end{cases}
\label{eq:alphad}
\end{equation}
where $L_0 = \frac{1}{2}\rho\vinf^2\sref \clo$ and $L_1 = \frac{1}{2}\rho\vinf^2\sref \cla$. The incident wind frame is then rotated around $\iy$ by $\alpha_d$ denoted by rotation matrix $\Ralpha$. The desired attitude can be calculated via consecutive rotations
\begin{equation}
    \Rd = \Rot \Ri \Ralpha \label{eq:Rd}
\end{equation}
The small angle approximation is often used implicitly for the fixed-wing aircraft control problems. In contrast, our approach explicitly uses the lift model in~\cref{eq:model_cl} to output the desired attitude, and relies on accurate estimation of the incident velocity $\vi$ with a dedicated sensor described in~\cref{ssec:3dsensor}.

%%%%%%%%%%%%%%%%%%%%%%%%%%%%%%%%%%%%%%%%%%%%%%%%%%%%%%%%%%%%%%%%%%%%%%%%%%%%%%%%
\subsection{Thruster Force Allocation}
At the current attitude, we calculate thruster force through
\begin{equation}
   \mask \fThat = \mask\left(\Rot^\top \fr - \fAhat\right) \label{eq:f_proj},
\end{equation}
where $\mask$ is the matrix masking out the component of force in the body axis that is unable to achieve by thruster input.$ \fAhat$ is estimated from~\cref{eq:fAhat,eq:alphabeta} using current $\vi$ measured from 3D airflow sensor. In the case of our vehicle from~\cref{fig:av_vtol}, the mask becomes $\mask =\diag{(1,0,1)}$ given the lack of \yaxis~thrust. The thrust commands $u_x$ and $u_z$ can be solved by substitute~\cref{eq:fThat} into~\cref{eq:f_proj}.

% The actual thruster commands $u_x$ and $u_z$ are solved by equating~\cref{eq:f_proj} and~\cref{eq:fThat}. The estimated external force becomes
% \begin{equation}
%     % \Rot\fbhat = 
%     \Rot \left( \fThat + \fAhat \right) = \Rot\left( \mask \Rot^\top \fr - \left(\mask - \B{I}\right)\fAhat\right)
%     \label{eq:fb_curr}
% \end{equation}

%%%%%%%%%%%%%%%%%%%%%%%%%%%%%%%%%%%%%%%%%%%%%%%%%%%%%%%%%%%%%%%%%%%%%%%%%%%%%%%%
\subsection{Force Model Composite Adaptation}
With reference force realized through~\cref{eq:alphad,eq:Rd,eq:f_proj}, the closed-loop dynamics of the velocity tracking error $\verr$ is
% Substitute~\cref{eq:f_ctrl,eq:fb_curr} into dynamics~\cref{eq:pos_dyn}, the closed-loop dynamics of the velocity tracking error $\verr$ is
\begin{equation}
\begin{split}
    \dverr &= \gravity + \Rot\fb - \dvr + \left(\fr - \fr\right) + (\Rot\fbhat - \Rot\fbhat) \\
    &= -\B{K}_v\verr + \left(\Rot\fbhat - \fr\right) - \Rot(\fbhat - \fb) \label{eq:verr_cl} \\
    % &= -\B{K}_v\verr + \Rot\big(\mask - \B{I}\big)\left(\Rot^\top \fr - \fAhat\right) - \Rot \basis \therr \\
    &= -\B{K}_v\verr + \ferr - \Rot \basis \therr.
\end{split}
\end{equation}
The force tracking error $\ferr  = \big(\Rot\fbhat - \fr\big)$ is affected by attitude difference $\Rerr$ and residual side force after projection in~\cref{eq:f_proj}. Any force error in the body $y$ direction cannot be directly compensated for by thrusters and has to be achieved through an attitude change.
% The force tracking error $\ferr  = \Rot\big(\mask - \B{I}\big)\left(\Rot^\top \fr - \fAhat\right)$ is affected by attitude difference and residual side force after projection in~\cref{eq:f_proj}. Any force error in the body $y$ direction cannot be directly compensated for by thrusters and has to be achieved through an attitude change.

The composite adaptation technique is employed to facilitate convergence of both tracking and prediction errors~\cite{slotine1989composite}. We compute prediction error via low-pass filtered onboard accelerometer measurement $\am = [\B{a}-\gravity]_f$ and basis calculation $\basisf = [\basis]_f$, where $[\cdot]_f$ here is a strictly-positive-real (SPR) filter~\cite{slotine1991applied,farrell2006adaptive}. Then the prediction error $\prederr$ is simply the difference between the two
\begin{equation}
    \prederr = \basisf \thhat - \am.
    \label{eq:prediction_error}
\end{equation}
We are now ready to define parameter update law:
\begin{align}
    \dot{\thhat} &= \cov\left(\basis^\top \Rot^\top \verr - \basisf^\top  \prederr \right) - \robustify \cov \left(\thhat - \thinit\right) \label{eq:th_adapt} \\
    \dot{\cov} &= -\cov \basisf^\top \basisf \cov + \forgetting \cov \label{eq:P_adapt}.
\end{align}
To improve the robustness, a damping term $\robustify$ is added to pull update toward initial estimates $\thinit$, which is obtained through model fitting described in~\cref{ssec:aerotest}. $\forgetting$ makes \cref{eq:P_adapt} exponentially forget data for continual adaptation~\cite{slotine1991applied}. 
% $\thinit$ is obtained through model fitting described in~\cref{ssec:aerotest}.

%%%%%%%%%%%%%%%%%%%%%%%%%%%%%%%%%%%%%%%%%%%%%%%%%%%%%%%%%%%%%%%%%%%%%%%%%%%%%%%%
\subsection{Stability Property}
From~\cref{eq:verr_cl}, it can be seen that the tracking of the velocity profile relies on the force prediction as well as attitude control. Due to the realistic aerodynamic effects on a fixed-wing VTOL being nonlinear and unsteady, we make the following assumption relating $\ferr$ and $\qerr$
\begin{assumption}
The vehicle is confined to operate within a linear force model range and $\norm{\ferr} \leq \C{F}\norm{\qerr}$, with $\C{F}$ being a constant related to maximum aerodynamic force.
\label{assm:ztoq}
\end{assumption}

% \xsedit{From~\cref{eq:fbar_solve}, it is clear that $\ferr = \Rot \fbhat(\Rot) - \Rd \fbhat (\Rd)$, which can be further manipulated into $\ferr = \Rd\big(\Rerr \fbhat(\Rd\Rerr) - \fbhat(\Rd) \big)$}

\begin{assumption}
The controller in~\cref{eq:m_ctrl} enables exponential tracking of $\angular$. The time-scale of $\angularerr$ convergence is much shorter than $\verr$ and $\qerr$ thus can be treated as $\angular = \angularref$.
\label{assm:wtowr}
\end{assumption}

With the composite adaptation on force model and tracking controller defined previously, the convergence of $\verr$, $\qerr$, and $\therr$ is given below.
\begin{theorem}
By applying the controller, force allocation, and model adaptation from~\cref{eq:f_ctrl,eq:m_ctrl,eq:Rd,eq:f_proj,eq:th_adapt,eq:P_adapt}, the tracking errors $\norm{\verr}$, $\norm{\qerr}$ and parameter error $\norm{\therr}$ will exponentially converge to a bounded error ball.
\end{theorem}

\begin{proof}
A Lyapunov function consisting of tracking error and prediction error terms is selected
\begin{equation}
    \C{V} = \frac{1}{2}\verr^\top \verr + \gamma^{-1} (2-2 \qoerr) + \frac{1}{2}\therr^\top \cov^{-1}\therr.
\end{equation}
where the global attitude tracking error $2-2\qoerr$ is used, similar to~\cite{lee2012exponential} for $SO(3)$, with its time derivative
\begin{equation*}
    \dqoerr = -\frac{1}{2}\qerr ^\top  \left(\angular - \Rerr^\top \angulard\right).
\end{equation*}
Taking the time derivative of $\C{V}$ and substituting in~\cref{eq:verr_cl,eq:th_adapt,eq:P_adapt} with~\cref{assm:wtowr}, we get
\begin{equation*}
\begin{split}
    % too complex, reduced steps for simplicity
    % \dot{\C{V}} &= \verr^\top \left( - \B{K} \verr + \ferr - \Rot \basis\therr\right)
    % + \gamma^{-1} \qerr^\top \Lambda_q \qerr \\
    % &\qquad + \frac{1}{2}\therr^\top \left(\basisf^\top \basisf - \lambda \cov\right)\therr \\
    % &\qquad + \therr^\top \left(\basis^\top \Rot^\top \verr - \basisf^\top  \prederr - \robustify  \left(\thhat - \thinit\right)\right)\\
    \dot{\C{V}} &= - \verr^\top\B{K}_v\verr + \verr^\top \ferr - \gamma^{-1}\qerr^\top \bm{\Lambda}_q \qerr \\
    & \qquad - \frac{1}{2}\therr^\top\left(\basisf^\top \basisf + \forgetting \cov \right)\therr \\
    & \qquad -\frac{\robustify}{2}\left(\norm{\therr}^2 + \norm{\thhat - \thinit}^2 - \norm{\bm{\theta} - \thinit}^2\right) 
\end{split}
\end{equation*}
We define the minimum eigenvalues of positive definite matrices, $c_v = \lambda_{\min}(\B{K}_v)$, $c_q = \lambda_{\min}(\B{\Lambda}_q)$ and $c_P = \lambda_{\min}(\B{P})$. Applying~\cref{assm:ztoq} and using the fact that $\basisf^\top\basisf \geq 0$, we arrive at
\begin{equation*}
\begin{split}
    \dot{\C{V}} &\leq - c_v \norm{\verr}^2 - \frac{c_q}{\gamma}\norm{\qerr}^2 - \frac{\forgetting c_P + \robustify}{2} \norm{\therr}^2 + \frac{\robustify}{2}\norm{\thinit}^2 \\
    & \qquad \qquad + \C{F}\norm{\verr}\norm{\qerr} \\
    &= -
    \begin{bmatrix}
    \norm{\verr}\\
    \norm{\qerr}
    \end{bmatrix}^\top
    \begin{bmatrix}
    c_v & \C{F}/2 \\
    \C{F}/2 & c_q/\gamma
    \end{bmatrix}
        \begin{bmatrix}
    \norm{\verr}\\
    \norm{\qerr}
    \end{bmatrix} - \frac{\forgetting c_P + \robustify}{2} \norm{\therr}^2 \\
    & \qquad \qquad + \frac{\robustify}{2}\norm{\bm{\theta} - \thinit}^2
\end{split}
\end{equation*}
By selecting a $\gamma < (4 c_v c_p)/\C{F}^2$, we can find a constant $c$, such that $\dot{\C{V}} < - c\C{V} + \frac{\robustify}{2}\norm{\thinit}^2$. This proves that $\norm{\verr}$, $\norm{\qerr}$ and $\norm{\therr}$ will exponentially converge to an error ball bounded by initial parameter error $\norm{\bm{\theta} - \thinit}$.
\end{proof}

\section{Model Fitting and Airflow Sensing}
\label{sec:aero}
\subsection{Aerodynamic and Thruster Coefficients}
\label{ssec:aerotest}
To accurately model the aerodynamic characteristics, the aircraft was mounted on a force sensor with all control surfaces in a neutral position. The data collected was for free-stream velocities up to \SI{9}{m/s}, and parameters were fit using Bayesian linear regression for the model described in~\cref{eq:model_cl,eq:model_cd}. Similarly, single motor and propeller combinations were tested to fit the model from~\cref{eq:model_T,eq:model_FS}. 

\begin{table}[b]
\renewcommand{\arraystretch}{1.3}
\centering
\caption{Fitted Parameters and Their Standard Deviations}
    \begin{tabular}{c c|c c}
    \hline
    {Thruster} & {mean $\pm$ std} & {Aero} & {mean $\pm$ std} \\
    \hline
    \hline
     $C_{Tx}$ & $\expnumber{3.02}{-3}\pm\expnumber{2.25}{-5}$  & $\clo$ & $0.3705\pm 0.06523$ \\
     $C_{Tz}$ & $\expnumber{2.87}{-3}\pm\expnumber{5.00}{-6}$ & $\cla$ & $3.2502\pm0.04434$ \\
     $C_{S}$  & $\expnumber{2.31}{-5}\pm\expnumber{1.18}{-5}$ & $\cdo$  & $0.1551\pm0.00394$\\
     $k_1$    & $1.425\pm0.010$ & $\cda$  & $0.1782\pm0.06518$ \\
     $k_2$    & $3.126\pm0.087$ & $\cdaa$ & $1.6000\pm0.10820$\\
    \hline
    \end{tabular}
    \label{tab:param_fit}
    % \vspace{-0.5cm}
\end{table}

To construct an empirically based model for rotor side force $F_S$, force data was collected for a propeller in forward flight at various angles of attack ranging from $-45^{\circ}$ to $45^{\circ}$ as shown in~\cref{fig:rotor_Drag_Data}. Dimensional analysis was then used to develop a relation. We assume that the side force depends on the freestream velocity $\vinf$, the air density $\rho$, the propeller's angular velocity $n$, and the propeller diameter $d$. Additionally, we expect some variation with $\alpha$, which is a dimensionless quantity. Per symmetry, there should be no side force at angles of attack of $\pm \ang{90}$. Data also indicates that $F_S$ achieves a maximum at an angle other than $\ang{0}$. Hence, we use a third order polynomial of $\alpha$ with zeros at $\pm \ang{90}$ to construct the model. These criteria combine to form the expression in~\cref{eq:model_Fs}. 

\begin{figure}[t]
\centering{
    \subfloat[Non-dimensionalized rotor side force data with numerical fit]{
        \includegraphics[width=.45\linewidth]{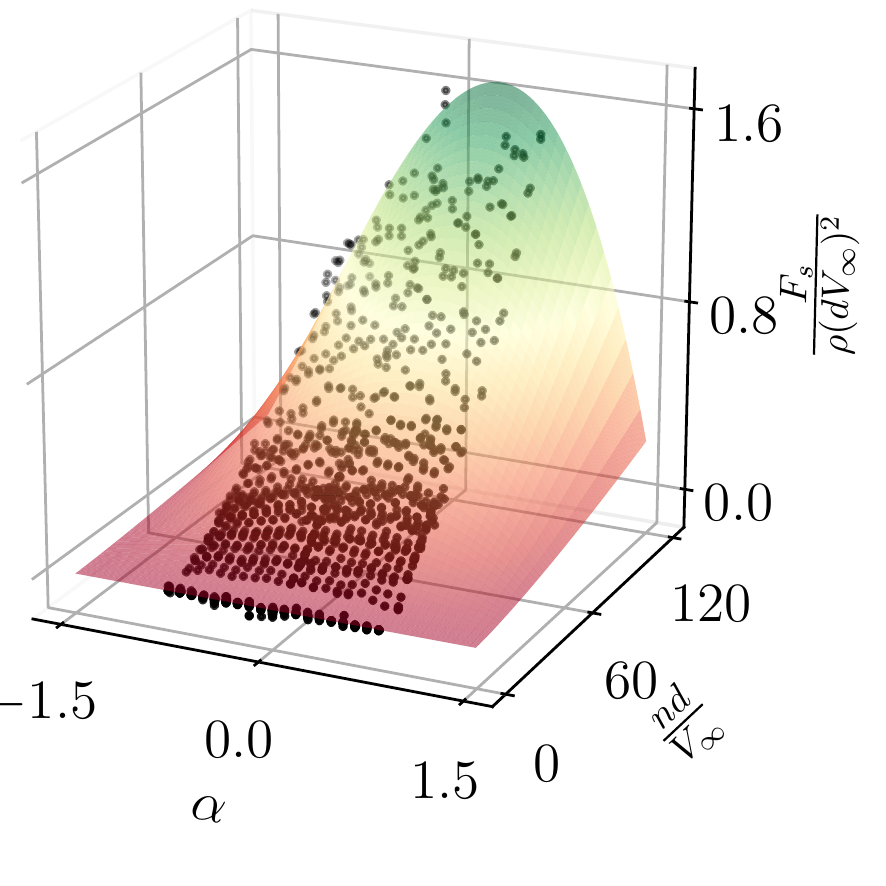}\label{fig:rotor_Drag_Data}
    }\hspace{0.1cm}
    \subfloat[Linearity of flow angle measurement]{
        \includegraphics[width=.45\linewidth]{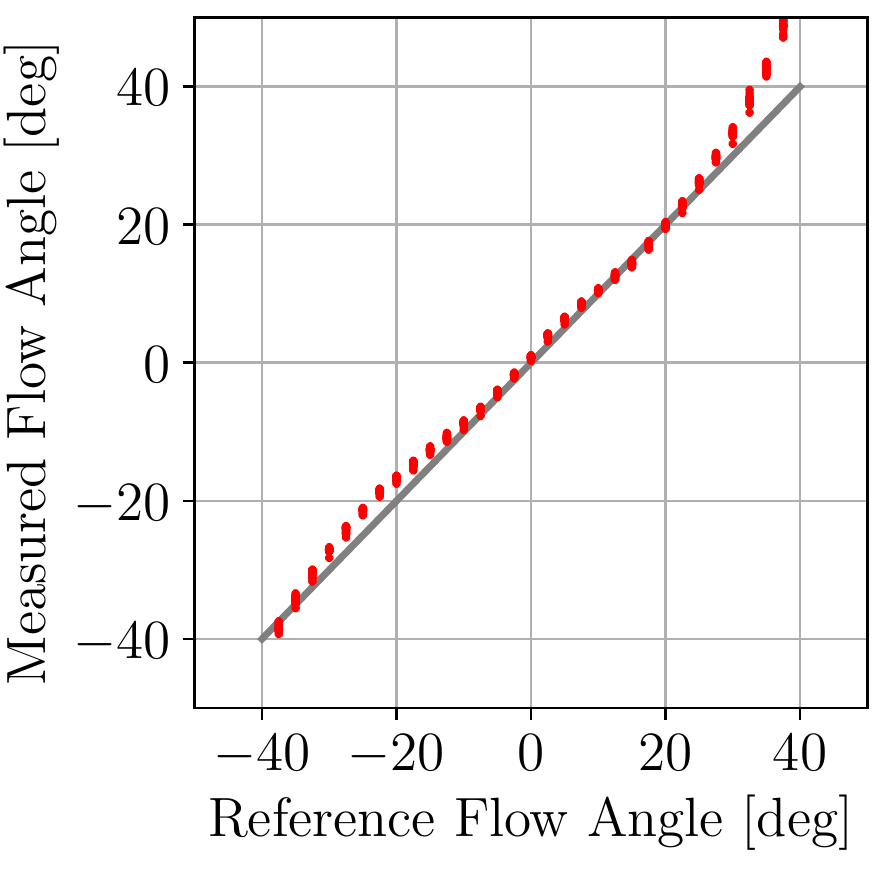}\label{fig:flow_sensor_linearity}
    }
}
\caption{Model fits for rotor side force and airflow vector sensor}
% \vspace{-6 mm}
\end{figure}

The results of the fits are summarized in~\cref{tab:param_fit}, split by thruster parameters and aerodynamic parameters. The standard deviation is from the posterior distribution of parameters after performing Bayesian linear regression using Markov Chain Monte-Carlo (MCMC) method.

\subsection{Airflow Sensing}
\label{ssec:3dsensor}

To measure the angle of attack ($\alpha$) and the sideslip angle ($\beta$) in flight, a custom 3D airflow sensor (\cref{fig:av_vtol}) containing three differential pressure sensors and a conic tip~\cite{popowski2015measurement} was developed. The sensor uses the SDP33 chip from Sensirion, a pressure sensor based on thermal mass flow~\cite{sensirion:sdp33}. These sensors have almost no zero-pressure offset and drift (\SI{0.2}{\pascal}), which makes them ideal for sensing flow angles. In addition, their fast response ($<\SI{3}{\milli\second}$) allows them to be used in the attitude control loop.  The conic tip has four orifices placed symmetrically relative to a central orifice. The difference in pressure for each pair of holes varies with $\alpha$ and $\beta$. The relationship between the differential pressures ($q_{\infty}, q_{\alpha}, q_{\beta}$) and the incident velocity $\vi$ in the body frame is modeled as:
\begin{equation}
\label{eq:aoa}
\vi =
\sqrt{\frac{2q_{\infty}/\rho}{q_{\infty}^2 + (k q_{\beta})^2 + (k q_{\alpha})^2 }}
% \frac{1}{\sqrt{q_{\infty}^2 + (k\cdot q_{\beta})^2 + (k\cdot q_{\alpha})^2 }} \sqrt{\frac{2q_{\infty}}{\rho}} 
\begin{bmatrix} 
q_{\infty}  \\
k q_{\beta} \\
k q_{\alpha} 
\end{bmatrix}
\end{equation}
% \begin{equation}
A wind tunnel was used to calibrate the sensor and fit the coefficient $k$ from the above equation. The linearity of estimated flow angle is shown in~\cref{fig:flow_sensor_linearity}.

The feedback from the sensor provides an accurate estimate of the incident airspeed vector $\vi$, which enables the aircraft to fly at desired aerodynamic conditions. The airspeed vector is used in the force model described by \cref{eq:fAhat,eq:fThat} in order to compensate for the aerodynamic forces and to provide an adaptation basis $\basis$.

% \label{eq:aoa}
% \vi =
% \frac{1}{\sqrt{q_{\infty}^2 + (k\cdot q_{\beta})^2 + (k\cdot q_{\alpha})^2 }} \sqrt{\frac{2q_{\infty}}{\rho}} \begin{bmatrix} 
% q_{\infty}  \\
% k q_{\beta} \\
% k q_{\alpha} 
% \end{bmatrix}
% \end{equation}.

% \begin{figure}[h!]
% \centering
% \includegraphics[width=0.5\linewidth]{figures/sideslip.pdf}
% \label{fig:tracking_comparison}
% \end{figure}

\section{Experiments}
\label{sec:exp}
We used a custom fixed-wing VTOL aircraft (\cref{fig:av_vtol}) running a modified PX4 firmware on a Pixhawk flight controller~\cite{meier2012pixhawk}. The properties of the aircraft are summarized in~\cref{tab:vtol_aircraft_properties}. We conducted indoor experiments using the Real Weather Wind  Tunnel from the Center for Autonomous Systems and Technologies (CAST). The fan array based wind tunnel can generate uniform wind fields up to \SI{12.9}{m/s} and produces wind speeds linearly with input throttle. The facility is equipped with motion capture cameras which allow closed-loop position control in front of the fan array.
\begin{figure}
\centering
\includegraphics[width=\linewidth]{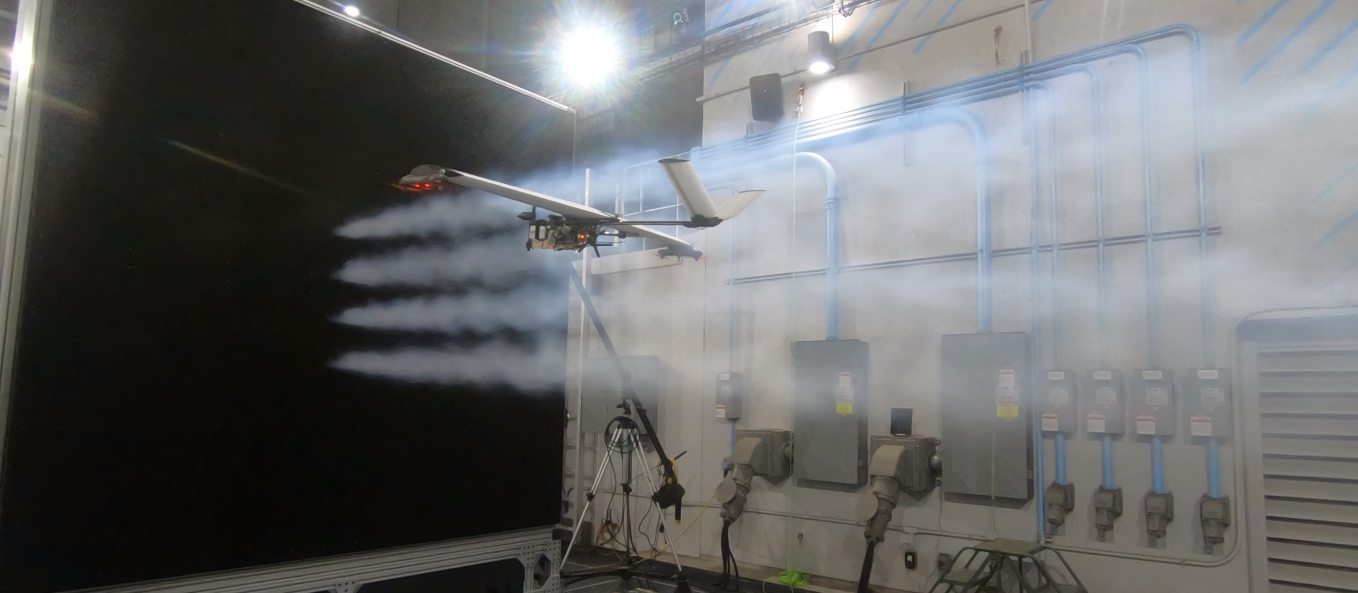}
\caption{The VTOL flying in position control in front of the CAST Fan Array, with smoke for flow visualization}
\label{fig:vtol_in_front_of_fan_array}
\end{figure}

\begin{table}[tb]
\renewcommand{\arraystretch}{1.3}
\centering
\caption{Fixed-wing VTOL aircraft properties}
\begin{tabular}{c|c||c|c}
\hline
Weight         & \SI{1.70}{kg}         & Lift Motor & T-motor F80-2200\\ 
% \hline
Wingspan       & \SI{1.08}{m}        & Front Motor    & T-motor AT2312-1150\\ 
% \hline
Wing Area      & \SI{0.223}{m^2}       & Lift Prop.     & King Kong 6"x4"              \\ 
% \hline
Battery & 6S \SI{1.3}{Ah} & Front Prop.         & APC 7"x4"    \\ \hline
\end{tabular}
\label{tab:vtol_aircraft_properties}
% \vspace{-4 mm}
\end{table}

During the experiment, the VTOL keeps its position fixed in front of the fan array, as shown in \cref{fig:vtol_in_front_of_fan_array}. The control loop as well as the force allocation and parameter adaptation is running at \SI{250}{Hz}. Given our estimate of $\thhat$'s posterior distribution from aerodynamic testing shown in~\cref{tab:param_fit}, we limit the maximum deviation of $\thhat$ to stay within fixed bounds around the initial parameter estimate. 
The filter used for the acceleration measurement and prediction \cref{eq:prediction_error} is a first order low-pass filter with a cutoff frequency of \SI{10}{Hz}. In all experiments, $\cov$ is initialized to a diagonal matrix ($\cov_0$) with elements listed in \cref{tab:adaptation_gains}.

\begin{table}[b]
    \renewcommand{\arraystretch}{1.3}
    \centering
    \caption{Initial adaptation gains: diagonal of $\cov_0$}
    \begin{tabular}{c c c c c c c c}
    % $\hat{{\theta}}_i$ 
    $\ctx$ & $\ctz$ & $\cfs$ & $\cdo$ & $\cda$ & $\cdaa$ & $\clo$ & $\cla$ \\
    \hline
    % $P_{0_{ii}}$ 
     $200$ & $200$ & $0.1$ & $1$ & $5$ & $20$ & $0.1$ & $0.1$ \\
    \hline
    \end{tabular}
    \label{tab:adaptation_gains}
    % \vspace{-0.6cm}
\end{table}

\subsection{Aerodynamic Parameters Convergence}

\cref{fig:parameter_convergence} shows parameter convergence for seven runs of the composite adaptation controller from~\cref{eq:th_adapt,eq:P_adapt}, each with randomly sampled initial parameters. Each experiment was started with the vehicle hovering in still air. The fan array was then sequentially set to 30\%, 50\%, and 70\% throttle, which roughly corresponds to \SI{4}{m/s}, \SI{6.5}{m/s} and \SI{9}{m/s} wind speed, as the vehicle continued to maintain its position. 

As the test conditions stay the same across different runs, and only initial parameter are varied, it can be seen from~\cref{fig:parameter_convergence} that the components of $\thhat$ converge to some patterns, indicated by the lower variance as time progresses. It is also interesting to note that $\clo$, $\cla$ and $\ctz$ have cleaner trends than the other variables, indicating steady and accurate force estimates in the vertical direction. There seem to be bigger variations in $\cda$ estimates. Overall, the parameter convergence is consistent with varying initial conditions, which reflects the robustness of the method.

\begin{figure}[t]
\centering
\includegraphics[width=\linewidth]{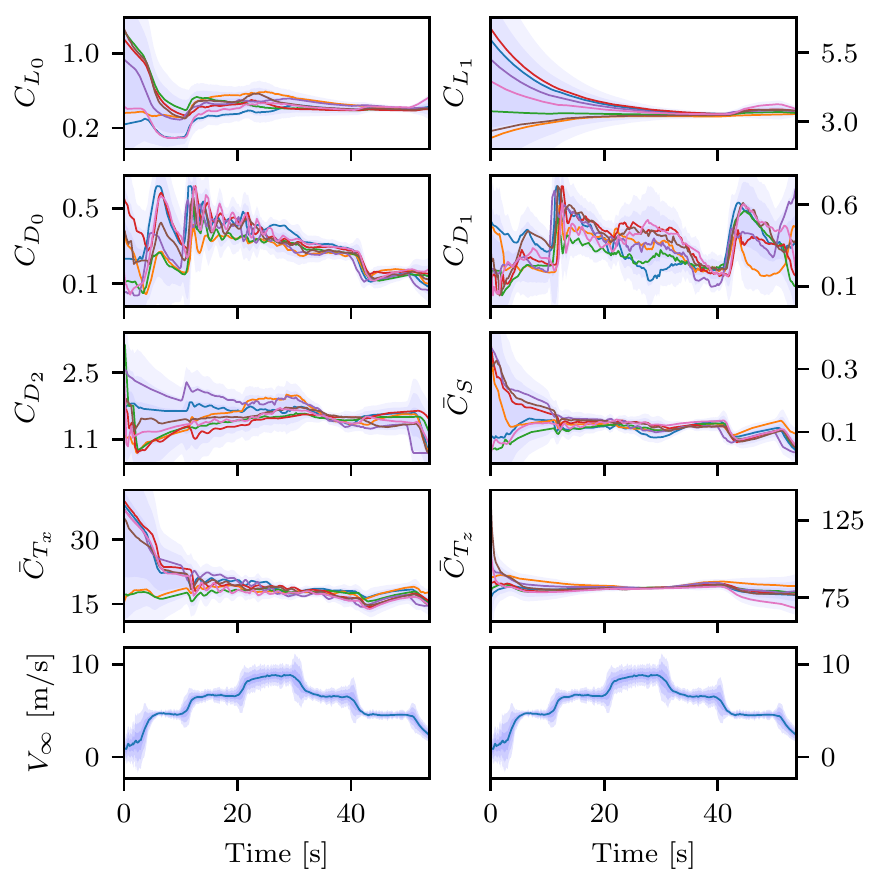}
\caption{Convergence of aerodynamic coefficients estimates for seven runs with random initial $\thhat$. The filled region represents 1$\sigma$ to 3$\sigma$ deviations. Freestream Velocity is duplicated for visualization}
\label{fig:parameter_convergence}
\vspace{-5 mm}
\end{figure}

\subsection{Comparison of Different Feedback Control Schemes}

\begin{figure}[t]
\centering
\includegraphics[width=\linewidth]{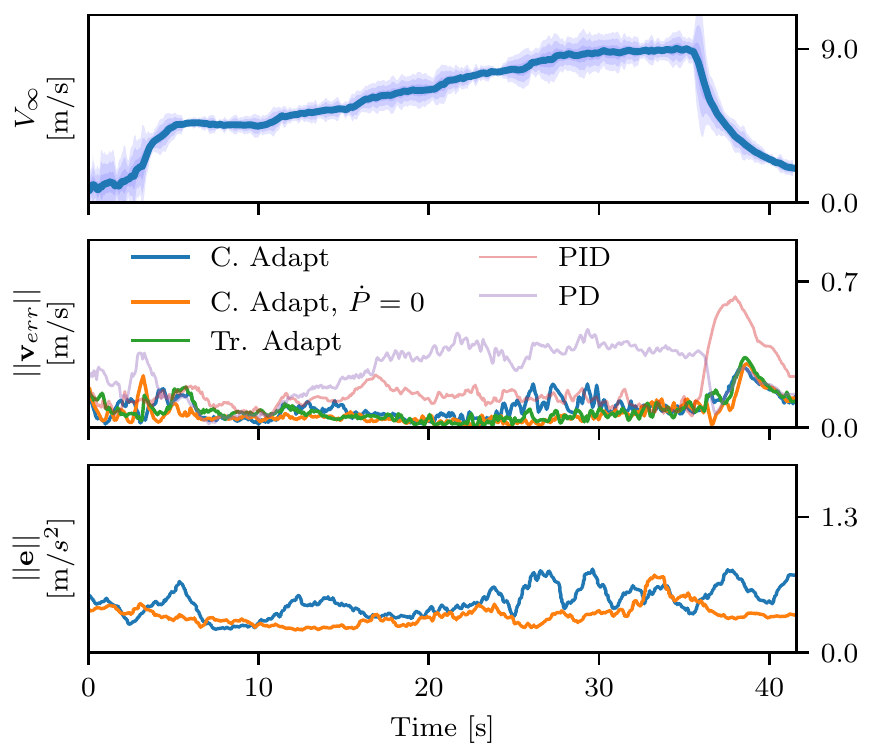}
\caption{From top to bottom, figures show measured total airspeed from 3d airflow sensor, norm of velocity tracking errors for all five controllers and norm of filtered (for visualization) prediction errors for the two composite adaptation controllers}
\label{fig:tracking_comparison}
\end{figure}

The tests started with vehicle in hover and the fan array off, the throttle is increased to 30\% (\SI{4}{m/s} wind) for 10 seconds and then to 70\% (\SI{9}{m/s} wind) in 10\% increments every 5 seconds. The fan array was switched off after 10 seconds at top throttle. The same experiment profile was used to test five different control and adaptation schemes: \textbf{(I)} The full implementation based on~\cref{eq:f_ctrl,eq:th_adapt,eq:P_adapt}, (C. Adapt); \textbf{(II)} Same method, but with $\cov$ update disabled (C. Adapt $\dot{\cov}=0$); \textbf{(III)}. Tracking error only adaptation, with 
$\dthhat = \cov_0 \left(\basis^\top\Rot^\top\verr\right)$; \textbf{(IV)} Proportional-Integral-Derivative (PID) controller; and \textbf{(V)} Proportional-Derivative (PD) controller.

The adaptive controllers (I) to (III) use the controller from~\cref{eq:f_ctrl}. And the PID and PD controller are based on Pixhawk's default implementation of multirotor control
$    \fr = -\gravity - \B{K}_P \verr - \B{K}_D \dverr -\B{K}_I \int_{0}^{t} \verr(\tau) \ d\tau %\label{eq:f_ctrl_pid}
$
with positive definite gain matrices $\B{K}_P$, $\B{K}_I$, $\B{K}_D$. The force allocation technique described in~\cref{sec:f_alloc} is still active with constant $\thhat$ using values from~\cref{tab:param_fit}. The equivalent gains used in all five controllers are held the same for consistent comparison. 

It is clear to see from~\cref{fig:tracking_comparison} that as airspeed increases, the PD controller started to accumulate unrecoverable large tracking errors. The PID controller can account for some model uncertainties and is almost on par with the adaptive controllers, however, we can see its slow response to changing disturbance as the vehicle overshoots significantly when airspeed ramps down.

\begin{figure}[t]
\centering
\includegraphics[width=\linewidth]{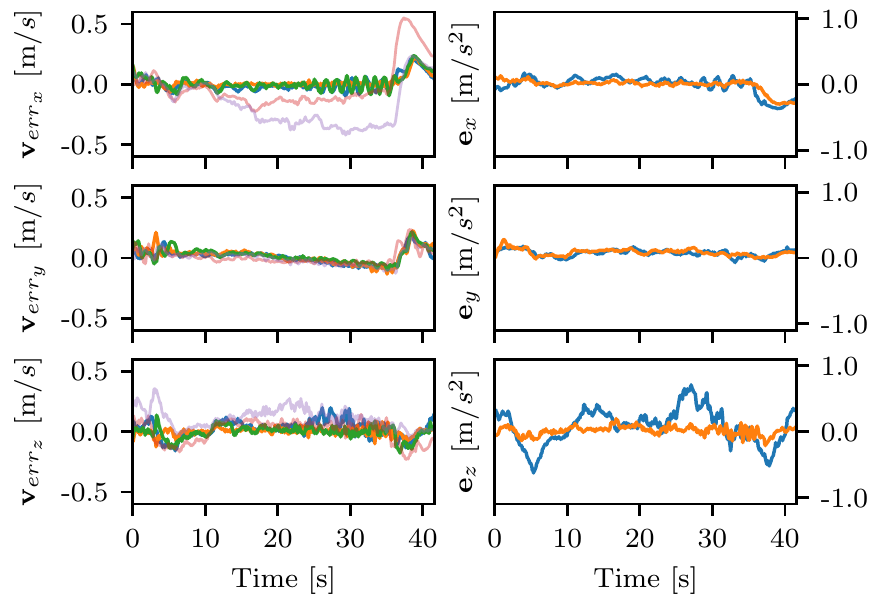}
\caption{Velocity tracking error and filtered (for visualization) acceleration prediction error of the same experiment as in \cref{fig:tracking_comparison}, for each axis individually. Velocity error is in Earth fixed North-East-Down frame with the vehicle flying north; acceleration is in body frame}
\label{fig:tracking_comparison_individual_axes}
\vspace{-4 mm}
\end{figure}

All versions of the adaptive controllers have better overall velocity tracking and do not suffer from the large overshoot that the PID has during sudden changes of the environment. For the two composite controllers, the prediction is kept as low as $0.05\gravity$ in error consistently. Interestingly, the constant gain composite controller actually outperforms the full implementation in terms of prediction accuracy. This is more prominent in~\cref{fig:tracking_comparison_individual_axes}, which shows the $\prederr$ for C. Adapt lies mostly in body \zaxis.

\section{Conclusions}
\label{sec:conclusion}
We have introduced a set of linear-in-parameter force models with good steady-state accuracy appropriate for fixed-wing VTOL. The 3D airflow sensor provides crucial realtime information related to the aerodynamic forces on the vehicle. The adaptation law, together with the controllers, is proven to possess the stability and robustness needed for high speed transition flight. After being integrated into our custom hybrid VTOL, the composite adaptive controller with 3D airflow sensor feedback greatly improves the tracking and prediction performance over baseline methods. 

\addtolength{\textheight}{-9cm}   % This command serves to balance the column lengths

%%%%%%%%%%%%%%%%%%%%%%%%%%%%%%%%%%%%%%%%%%%%%%%%%%%%%%%%%%%%%%%%%%%%%%%%%%%%%%%%

%%%%%%%%%%%%%%%%%%%%%%%%%%%%%%%%%%%%%%%%%%%%%%%%%%%%%%%%%%%%%%%%%%%%%%%%%%%%%%%%


%\section*{Acknowledgment}
%We thank M. Gharib, Director of CAST at Caltech, for his support and technical guidance. 

%%%%%%%%%%%%%%%%%%%%%%%%%%%%%%%%%%%%%%%%%%%%%%%%%%%%%%%%%%%%%%%%%%%%%%%%%%%%%%%%
\newpage
\bibliographystyle{IEEEtran}
\bibliography{IEEEabrv,ref}

\end{document}